\documentclass[letterpaper, 10 pt, conference]{ieeeconf}  
\usepackage{graphicx}
\usepackage{cite}
\usepackage{amsmath}
\usepackage{amssymb}
\usepackage{mathtools}
\usepackage{upgreek}
\usepackage{multirow} 
\usepackage{threeparttable}
\usepackage{booktabs}
\usepackage[caption=false,font=footnotesize]{subfig}
\usepackage{array}

\usepackage{amsthm}

\theoremstyle{plain}
\newtheorem{theorem}{Theorem}
\newtheorem{lemma}{Lemma}

\theoremstyle{definition}
\newtheorem{definition}{Definition}

\theoremstyle{remark}
\newtheorem{remark}{Remark}

\makeatletter
\g@addto@macro\thm@space@setup{%
  \thm@preskip=2pt
  \thm@postskip=2pt
}
\makeatother

\IEEEoverridecommandlockouts                              

\title{\LARGE \bf
A Feasibility-Enhanced Control Barrier Function Method for  Multi-UAV Collision Avoidance
}

\author{Qishen Zhong, Junlong Wu, Jian Yang$^*$, Guanwei Xiao, Junqi Wu, Zimeng Jiang and Pingan Fang
\thanks{* Corresponding Author.}
\thanks{All Authors  are with the School of Automation Science and Engineering, South China University of Technology,
Guangzhou 510641, Guangdong, China}%
}

\begin{document}

\maketitle
\thispagestyle{empty}
\pagestyle{empty}

\begin{abstract}

This paper presents a feasibility-enhanced control barrier function (FECBF) framework for multi-UAV collision avoidance.
In dense multi-UAV scenarios, the feasibility of the CBF quadratic program (CBF-QP) can be compromised due to internal incompatibility among multiple CBF constraints. To address this issue, we analyze the internal compatibility of CBF constraints and derive a sufficient condition for internal compatibility.
Based on this condition, a sign-consistency constraint is introduced to mitigate internal incompatibility.
The proposed constraint is incorporated into a decentralized CBF-QP formulation using worst-case estimates and slack variables.
Simulation results demonstrate that the proposed method significantly reduces infeasibility and improves collision avoidance performance compared with existing baselines in dense scenarios. Additional simulations under varying time delays demonstrate the robustness of the proposed method.
Real-world experiments validate the practical applicability of the proposed method.

\end{abstract}

\section{INTRODUCTION}
Multi-UAV systems have been widely employed in various applications, including search and rescue\cite{ge2025multi}, inspection\cite{cao2025cooperative}, agriculture\cite{zhang2025fast}, and cargo delivery~\cite{cattai2025multi}. 
When multiple UAVs fly simultaneously in shared airspace, the risk of inter-UAV collisions poses a major challenge to operational safety, making effective collision avoidance a fundamental requirement in such systems.

Various approaches have been proposed to address multi-UAV collision avoidance, including geometry-based methods~\cite{yang2025geometry, qin2023srl, arul2021v}, artificial potential field methods~\cite{yang2023collision, qian2025collision}, and deep reinforcement learning methods~\cite{zhong20253d, han2025deep, kuo2025deep}.
However, while these methods have demonstrated effectiveness, establishing formal safety guarantees for them remains nontrivial, which may restrict their use in strictly safety-critical settings. Control barrier function (CBF)-based methods~\cite{huang2025dynamic, liu2025safety, huriot2025safe} provide a principled framework for enforcing safety constraints with formal guarantees, and have therefore attracted increasing attention.

Despite their appealing theoretical properties, existing CBF-based methods still face significant feasibility challenges in dense multi-UAV scenarios. As the number of UAVs increases, a large number of safety constraints must be satisfied simultaneously, which may cause the feasible control set to shrink or even become empty, especially in large-scale collision avoidance scenarios. In such cases, the underlying CBF quadratic program (CBF-QP) becomes infeasible, and no admissible control input can be obtained.

Several studies have investigated feasibility related issues about CBF-based methods. 
Xiao \textit{et al.}~\cite{xiao2022sufficient} derived sufficient conditions to ensure the compatibility between CBF constraints and control bounds, thereby guaranteeing per-step feasibility under bounded inputs. 
Their analysis implicitly assumes that multiple CBF constraints are internally compatible.
Isaly \textit{et al.}~\cite{isaly2024feasibility} developed a unified theoretical framework for analyzing feasibility and continuity of optimization-based feedback controllers. It provides tools to certify feasibility but without modifying the underlying constraints to improve it. Neither of these approaches explicitly improves the internal compatibility among multiple CBF constraints.

In multi-UAV collision avoidance problem, the internal compatibility among multiple CBF constraints becomes significantly more difficult to satisfy, and it is a prerequisite for feasibility. 
If multiple CBF constraints are mutually incompatible, ensuring compatibility with control bounds or certifying feasibility loses practical relevance, since no admissible control input exists regardless of the available control authority. 
Therefore, improving the internal compatibility among CBF constraints is essential for enhancing feasibility in dense multi-UAV scenarios, but remains insufficiently addressed by existing CBF-based approaches.

\begin{figure}[!tb] 
  \centering
  
    \includegraphics[width=0.47\textwidth]{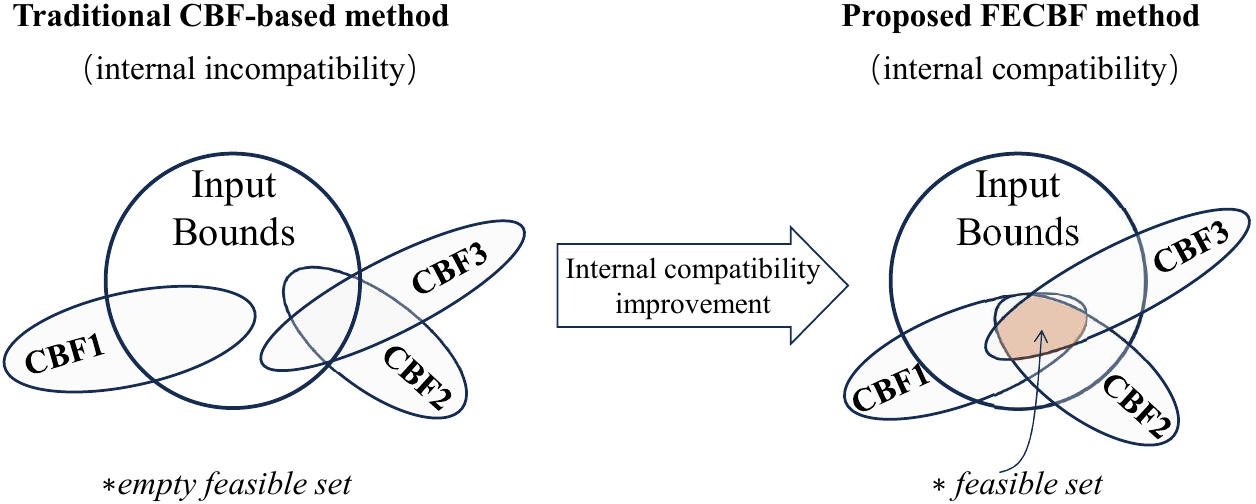}
  
  \caption{Illustration of internal compatibility and incompatibility among multiple CBF constraints in multi-UAV collision avoidance. Compatibility refers to the existence of at least one control input that simultaneously satisfies all CBF constraints within the admissible input bounds. Internal compatibility specifically denotes the mutual consistency of the CBF constraints themselves, meaning that their intersection is non-empty without considering the admissible input bounds. Each ellipse represents a CBF constraint, while the circle denotes the admissible input bounds. }
  \label{fig:intro}

\end{figure}

This letter proposes a feasibility-enhanced CBF (FECBF) method for multi-UAV collision avoidance, aiming to improve the feasibility of CBF-based safety controllers. By explicitly examining the internal compatibility among multiple CBF constraints, the proposed method addresses a key prerequisite for the feasibility of the CBF-QP, as illustrated in Fig.~\ref{fig:intro}. Specifically, we first investigate the structural conditions under which multiple CBF constraints become internally incompatible and derive a sufficient condition for compatibility. Based on this condition, we further design a sign-consistency constraint that can be incorporated into the CBF-QP to guide the control inputs toward satisfying the compatibility condition. 
The resulting formulation enhances the feasibility of the CBF-QP while preserving the safety guarantees of CBFs, and can be implemented in decentralized multi-UAV systems using only local interaction information. The main contributions of this work are summarized as follows:
\begin{itemize}
\item We analyze the internal compatibility among multiple CBF constraints and derive a sufficient condition for internal compatibility.
Since internal compatibility is a necessary prerequisite for the feasibility of the CBF-QP, satisfying the proposed condition directly improves feasibility and effectively alleviates infeasibility in practice, even without explicitly modeling the relationship between the CBF constraints and the control bounds.
\item Based on the derived sufficient condition, we design a sign-consistency constraint for the CBF-QP that guides the control inputs toward satisfying this condition. This constraint enables the proposed feasibility-enhancing mechanism to be implemented in decentralized multi-UAV systems using local interaction information.
\end{itemize}

The rest of this letter is organized as follows: Section \ref{sepre} presents the preliminaries and problem formulation. Section \ref{met} elaborates on the proposed method. Section \ref{se3} presents numerical simulations and real-world experiments. Finally, Section \ref{se5} concludes the letter.

\section{PRELIMINARY}
\label{sepre}
\subsection{Notations}
In this subsection, we summarize the main notations used in this letter.
Let $\mathbb{Z}_{\ge 0}$ and $\mathbb{Z}_{>0}$ denote the sets of nonnegative integers and positive integers, respectively.
For $n \in \mathbb{Z}_{>0}$, the notation $[n] := \{1,\ldots,n\}$ denotes an index set of size $n$, and we define $[0] := \emptyset$.
The symbol $\lVert \cdot \rVert$ denotes the Euclidean norm of a vector. Column vectors are denoted by lowercase boldface letters such as $\mathbf{x}$, while matrices are denoted by uppercase boldface letters such as $\mathbf{W}$. The superscript $\top$ denotes the transpose operation. For two vectors $\mathbf{x}$ and $\mathbf{y}$, the relation $\mathbf{x} \ge \mathbf{y}$ means that each element of $\mathbf{x}$ is greater than or equal to the corresponding element of $\mathbf{y}$, and $\mathbf{x} \neq \mathbf{y}$ means that at least one element of $\mathbf{x}$ is different from the corresponding element of $\mathbf{y}$. $\widetilde{\mathbf{x}}$ denotes the normalization of vector $\mathbf{x}$.

\subsection{Control Barrier Function}
Consider a class of first-order nonlinear control-affine systems described by
\begin{equation}
\label{Eq1} 
\dot{\mathbf{x}} = f(\mathbf{x}) + g(\mathbf{x})\mathbf{u},
\end{equation}
where $\mathbf{x} \in \mathcal{X} \subset \mathbb{R}^n$ is the state, 
$\mathbf{u} \in \mathcal{U} \subset \mathbb{R}^m$ is the control input, 
$f : \mathcal{X} \to \mathbb{R}^n$ and $g : \mathcal{X} \to \mathbb{R}^{n \times m}$ 
are locally Lipschitz continuous functions. 
The safe set $\mathcal{C} \subset \mathcal{X}$ is defined as the set of states that should remain safe and is characterized by a continuously differentiable function $h : \mathcal{X} \to \mathbb{R}$ as
\begin{align}
\label{Eq2} 
\mathcal{C} \coloneqq \{\mathbf{x} \in \mathcal{X} :h(\mathbf{x})\geq0\},
\quad\partial\mathcal{C} \coloneqq \{\mathbf{x} \in \mathcal{X}:h(\mathbf{x})=0\}.
\end{align}
where $\partial\mathcal{C}$ denotes the boundary of $\mathcal{C}$.
\begin{definition}[CBF,  \cite{ames2016control}]
Given the system \eqref{Eq1} and the safe set \eqref{Eq2}, the function $h$ is a CBF if $\nabla h(\mathbf{x})\neq0$ for all $\mathbf{x}\in\partial\mathcal{C}$ and there exists an extended class \( \mathcal{K}_{\infty} \) function \( \alpha \) such that 
\begin{equation}
\underset{\mathbf{u}\in\mathcal{U}}{\sup}[L_f h(\mathbf{x}) + L_g h(\mathbf{x}) \mathbf{u} + \alpha(h(\mathbf{x}))] \geq 0,
\label{Eq3}
\end{equation}
for all $\mathbf{x}$ in an open set $\mathcal{D}\supset\mathcal{C}$. \({L}_f h\) and \({L}_g h\) are the Lie derivatives along the flows of \(f(\mathbf{x})\) and \(g(\mathbf{x})\), respectively. 
\end{definition}
\begin{lemma}
Given a control barrier function $h$, any locally Lipschitz continuous controller $k : \mathcal{X} \to \mathbb{R}^m$ such that the control input $\mathbf{u} = k(\mathbf{x})$ satisfies
\begin{equation}
\dot h(\mathbf x)+\alpha(h(\mathbf x))\ge 0,\quad \forall\,\mathbf x\in\mathcal C,
\label{eq:cbf_condition}
\end{equation}
renders the set $\mathcal{C}$ forward invariant, where 
\begin{equation}
\dot h(\mathbf x):= L_f h(\mathbf x)+L_g h(\mathbf x)k(\mathbf x),
\label{eq:cbf_condition2}
\end{equation}
denotes the time derivative of $h$ along the closed-loop system trajectories.
\end{lemma}

Forward invariance means that for any initial state $\mathbf{x}_0\coloneqq \mathbf{x}(t_0) \in \mathcal{C}$, the trajectory $\mathbf{x}(t)$ remains in $\mathcal{C}$ for $t\geq t_0$ under the input $\mathbf{u}=k(\mathbf{x})$.
\subsection{Unmanned Aerial Vehicle Kinematic Model}
Consider $n$ UAVs operating in a shared Three-Dimensional (3-D) environment, where A$_i$ denotes the $i$-th UAV, $i\in[n]$. The position and velocity of A$_i$ with respect to the global inertial frame $\rm O-XYZ$ are denoted as $\mathbf{p}_i=\left[p_{ix}, p_{iy}, p_{iz}\right]^{\top}$ and $\mathbf{v}_i=\left[v_{ix}, v_{iy}, v_{iz}\right]^{\top}$, respectively. The UAV kinematics are given by
\begin{align}
\begin{split}
\label{Eq5} 
\begin{cases}
\dot{p}_{ix}=v_{ix}=v_i\cos\theta_i\cos\psi_i\\
\dot{p}_{iy}=v_{iy}=v_i\cos\theta_i\sin\psi_i \\
\dot{p}_{iz}=v_{iz}=v_i\sin\theta_i\\
\dot{v}_{i}=a_i, \quad 
\dot{\theta}_{i}=\gamma_i,\quad
\dot{\psi}_{i}=\omega_i
\end{cases},
\end{split}
\end{align}
where $v_i\coloneqq \lVert \mathbf{v}_i \rVert$ is the speed, $\theta_i$ is the pitch angle, $\psi_i$ is the yaw angle, $a_i$ is the linear acceleration, $\gamma_i$ and $\omega_i$ are the angular velocities of pitch and yaw, respectively. 

To account for the practical maneuverability, the input and state constraints can be expressed as
\begin{align}
\begin{split}
\label{Eq6} 
\begin{cases}
\theta^{\rm min}_i\leq \theta_i\leq\theta^{\rm max}_i, \quad\quad\psi^{\rm min}_i\leq \psi_i\leq \psi^{\rm max}_i\\
v^{\rm min}_i\leq v_i\leq v^{\rm max}_i, \quad\quad a^{\rm min}_i\leq a_i\leq a^{\rm max}_i \\
\gamma^{\rm min}_i\leq \gamma_i\leq \gamma^{\rm max}_i, \quad\quad\omega^{\rm min}_i\leq \omega_i\leq \omega^{\rm max}_i\\
\end{cases},
\end{split}
\end{align} 
where $(\cdot)^{\rm min}$ and $(\cdot)^{\rm max}$ denote the lower and upper bounds of the corresponding variables. 

\subsection{Problem Formulation}

To guarantee navigation safety, each UAV A$_i$ is required to maintain a safety region with radius $r_i$ around its position.
The positions of A$_i$ and A$_j$ are denoted by $\mathbf{p}_i, \mathbf{p}_j \in \mathbb{R}^3$.
From the perspective of A$_i$, the collision avoidance requirement is expressed as
\begin{equation}
\label{Eq7}
\|\mathbf{p}_i-\mathbf{p}_j\| > r_{i}, \quad \forall j\in[n]\backslash\{i\}.
\end{equation}

In this letter, we aim to design a decentralized control strategy for each UAV A$_i$ such that the constraint \eqref{Eq7} is satisfied using its local control input $\mathbf{u}_i=[a_i,\gamma_i,\omega_i]^{\top}$ during the navigation task.

\section{METHOD}
\label{met}
In this section, we first formulate a centralized CBF-QP model for multi-UAV collision avoidance in 3-D environments. This formulation is used to analyze the internal compatibility among multiple CBF constraints and identify key factors contributing to infeasibility. Motivated by this analysis, a sign-consistency constraint is proposed to improve compatibility. Finally, the proposed constraint is incorporated into a decentralized CBF-QP framework, where each UAV computes its collision avoidance control using only local state information, without requiring the control inputs of neighboring UAVs.

\subsection{Centralized CBF-QP}
Collision avoidance is ensured if each UAV satisfies the safety requirement \eqref{Eq7}. 
This requirement can be guaranteed by ensuring the forward invariance of the corresponding safe set via CBFs.
However, when $\mathbf{p}_i$ is used as the system state, the control input does not directly affect the state derivative, resulting in a relative degree larger than one with respect to distance-based safety constraints. 
As a consequence, standard CBF conditions cannot be directly applied.

To address this issue, we introduce a virtual state variable $\mathbf{s}_i:=\mathbf{p}_i + \zeta \mathbf{v}_i$, where $\zeta > 0$ is a design parameter. The virtual state dynamics are expressed as
\begin{subequations}
\label{Eq8}
\begin{align}
\dot{\mathbf{s}}_i &= \dot{\mathbf{p}}_i + \zeta \dot{\mathbf{v}}_i 
                   = \mathbf{v}_i + \zeta \mathbf{W}_i \mathbf{u}_i, \\
\mathbf{W}_i &= \mathbf{R}_i \operatorname{diag}(1,\, v_i,\, v_i),
\end{align}
\end{subequations}
where $\mathbf{W}_i$ is the kinematic Jacobian matrix and $\mathbf{R}_i$ is the rotation matrix given by

\begin{align}
\label{Eq9}
\mathbf{R}_i =
\begin{bmatrix}
\cos\theta_i\cos\psi_i & -\sin\theta_i\cos\psi_i & -\cos\theta_i\sin\psi_i \\
\cos\theta_i\sin\psi_i & -\sin\theta_i\sin\psi_i & \cos\theta_i\cos\psi_i \\
\sin\theta_i           & \cos\theta_i            & 0
\end{bmatrix}.
\end{align}

With the virtual state dynamics in \eqref{Eq8}, the resulting system can be written in a first-order control-affine form as in \eqref{Eq1}, which enables the direct construction of CBF constraints.

For any two UAVs A$_i$ and A$_j$, a safe set $\mathcal{C}_{i,j} \coloneqq \{(\mathbf{s}_i,\mathbf{s}_j)\mid h_{i,j}(\mathbf{s}_i,\mathbf{s}_j)\geq0\}$ is defined using the CBF
\begin{align}
\label{Eq10} h_{i,j}=\lVert \mathbf{s}_i-\mathbf{s}_j \rVert^2-\left(d_{i,j}\right)^2,
\end{align}
where $d_{i,j}=r_i+r_j+\zeta(v_i+v_j)$. The definition of $d_{i,j}$ introduces a geometric safety margin beyond the physical size of each vehicle. By substituting Eq.~\eqref{Eq10} into \eqref{eq:cbf_condition}, the pairwise CBF constraint for two UAVs is expressed as
\begin{align}
\begin{split}
\label{Eq11} \dot{h}_{i,j}+\alpha(h_{i,j})=\mathbf{k}^{\top}_{i,j}\mathbf{u}_i+\mathbf{k}^{\top}_{j,i}\mathbf{u}_j+\xi_
{i,j}\geq0, 
\end{split}
\end{align}
where $\xi_{i,j}$ collects all terms independent of the control inputs, and
\begin{align}
  \begin{split}
\label{Eq12} 
    \mathbf{k}_{i,j}^\top &= 2\zeta (\mathbf{s}_i - \mathbf{s}_j)^\top \mathbf{W}_i, 
    \quad\mathbf{k}_{j,i}^\top = 2\zeta (\mathbf{s}_j - \mathbf{s}_i)^\top \mathbf{W}_j, \\
    \xi_{i,j} &= 2(\mathbf{s}_i - \mathbf{s}_j)^\top (\mathbf{v}_i - \mathbf{v}_j) + \alpha(h_{i,j}). \\
    \alpha(h_{i,j}) &= \kappa h_{i,j}, \quad \kappa>0.
  \end{split}
\end{align}

\begin{remark}[\textit{Velocity-dependent safety margin}]
A$_i$ is assumed to occupy a safety region of radius $r_i$ as defined in \eqref{Eq7}.
In the CBF formulation, collision avoidance is enforced by maintaining a separation distance
no smaller than $r_i+r_j$ between two UAVs.
On top of this geometric margin, the safety distance $d_{i,j}$ further incorporates
velocity-dependent terms to account for braking capability and maneuverability limits.
Although $d_{i,j}$ depends on the velocities $v_i$ and $v_j$, its time derivative is not
explicitly included in the CBF constraint \eqref{Eq11}.
This design preserves modeling simplicity and analytical tractability.
The velocity-dependent component is therefore treated as an additional conservative margin
rather than a dynamic state variable.
Moreover, the presence of the geometric safety margin ensures that bounded variations in the
velocity-dependent term do not compromise the underlying collision-free requirement.
\end{remark}

In multi-UAV scenarios, multiple pairwise CBF constraints \eqref{Eq11} are imposed to ensure safety for every pair of UAVs. Based on these constraints, we formulate a centralized CBF-QP to compute collision-free control inputs for all UAVs. Define $\mathbf{u}^{p}_i\in\mathbb{R}^{3}$ as the nominal control input according to the navigation objective  of A$_i$. With $\mathbf{u}=[\mathbf{u}^{\top}_1,\dots,\mathbf{u}^{\top}_n]^{\top}$ and $\mathbf{u}^{p}=[\mathbf{u}^{p \top}_1,\dots,\mathbf{u}^{p \top}_n]^{\top}$, the CBF-QP is expressed as 
\begin{subequations}
\label{Eq15} 
\begin{align}
\mathop{\arg \min}_{\mathbf{u}} & \qquad \lVert\mathbf{u}-\mathbf{u}^{p} \rVert^2 \label{Eq15a}\\
\rm {s.t.}  & \quad \mathbf{u}^{\min} \leq \mathbf{u} \leq \mathbf{u}^{\max},\label{Eq15b}\\
  & \qquad\quad \mathbf{C}\mathbf{u} \leq\mathbf{b} \label{Eq15c}, 
\end{align}
\end{subequations}
where $\mathbf{u}^{\min}, \mathbf{u}^{\max}\in \mathbb{R}^{3n}$ denote the lower and upper bounds of the control inputs for all UAVs.  The matrix $\mathbf{C}\in\mathbb{R}^{\frac{n(n-1)}{2}\times3n}$ and the vector $\mathbf{b}\in\mathbb{R}^{\frac{n(n-1)}{2}}$ are constructed by stacking all pairwise CBF constraints in \eqref{Eq16}.

\begin{align}
\begin{split}
\label{Eq16} 
\mathbf{C}=&\begin{bmatrix}
  \mathbf{k}^{\top}_{1,2} & \mathbf{k}^{\top}_{2,1} & \mathbf{0}^\top & \mathbf{\cdots} & \mathbf{0}^\top & \mathbf{0}^\top\\
  \mathbf{k}^{\top}_{1,3} & \mathbf{0}^\top & \mathbf{k}^{\top}_{3,1} & \cdots & \mathbf{0}^\top &\mathbf{0}^\top\\
  \vdots & \vdots & \vdots & \ddots & \vdots &\vdots\\
  \mathbf{0}^\top & \mathbf{k}^{\top}_{2,3} & \mathbf{k}^{\top}_{3,2} & \cdots & \mathbf{0}^\top &\mathbf{0}^\top\\ 
  \vdots & \vdots & \vdots & \ddots & \vdots &\vdots\\
  \mathbf{0}^\top & \mathbf{0}^\top & \mathbf{0}^\top & \cdots & \mathbf{k}^{\top}_{n-1,n} &\mathbf{k}^{\top}_{n,n-1}
\end{bmatrix}\\
\mathbf{b}=&\begin{bmatrix}
  \xi_{1,2} & \xi_{1,3} & \cdots & \xi_{2,3} & \cdots & \xi_{n-1,n}
\end{bmatrix}^\top
\end{split}.
\end{align}

\subsection{Feasibility Analysis and Sign-Consistency Constraint}
As shown in the previous subsection, we could obtain the collision avoidance solutions by solving the centralized CBF-QP as defined in \eqref{Eq15}.
However, if the feasible set becomes empty, the constraints in \eqref{Eq15} cannot be satisfied simultaneously, and collision avoidance cannot be guaranteed.

The feasibility of the CBF-QP depends on two types of compatibility: 
internal compatibility among the CBF constraints, 
and overall compatibility considering both the CBF constraints and the input bounds. 
Internal compatibility means that there exists a control input $\mathbf u$ satisfying \eqref{Eq15c}, 
while overall compatibility requires that there exists a control input $\mathbf u$ satisfying both \eqref{Eq15b} and \eqref{Eq15c}.
Internal compatibility is a necessary condition for feasibility.
If it is violated, the CBF-QP becomes infeasible regardless of the admissible input range. When both types of compatibility are satisfied, they jointly provide a sufficient condition for feasibility. The intersection of the CBF constraints is likely to become empty when multiple UAVs fly in dense  environments, which is a fundamental bottleneck limiting the feasibility of the CBF-QP. Therefore, we focus on improving the internal compatibility of the CBF constraints.

By introducing slack variables, the classical Farkas' lemma \cite{bertsimas1997introduction} is extended to Lemma \ref{the2}, which provides a criterion for determining the internal compatibility of the CBF constraints.
\begin{lemma}
  \label{the2}
Given the CBF constraint \eqref{Eq15c}, exactly one of the following two alternatives holds:

(a) There exists a $\mathbf{u}$ such that $\mathbf{Cu}\leq\mathbf{b}$.

(b) There exists a vector $\mathbf{q}\geq\mathbf{0}$ such that $\mathbf{q}^\top\mathbf{C}=\mathbf{0}^\top$ and $\mathbf{q}^\top\mathbf{b}<0$.
\end{lemma}

According to Lemma~\ref{the2}, the CBF constraints are incompatible if and only if case~(b) holds. Using this compatibility criterion, Lemma~\ref{pro1} further establishes the relationship between internal compatibility and the UAV count $n$. 

\begin{lemma}
\label{pro1}
For $n \ge 7$, the dimension of the left nullspace of $\mathbf{C}$ admits a quadratic lower bound in $n$, which increases the potential for internal incompatibility among the CBF constraints.
\end{lemma}

\begin{proof}
According to case~(b) of Lemma~\ref{the2}, the condition $\mathbf{q}^\top\mathbf{C} = \mathbf{0}^\top$ implies that $\mathbf{q}$ belongs to the left nullspace of $\mathbf{C}$, denoted by $\mathcal{N}(\mathbf{C}^\top)$. The dimension of $\mathcal{N}(\mathbf{C}^\top)$ is given by
\begin{align}
  \label{Eq17}
  \dim(\mathcal{N}(\mathbf{C}^\top)) = \frac{n(n-1)}{2} - \operatorname{rank}(\mathbf{C}).
\end{align}

For $n \ge 7$, we have $n(n-1)/2 \ge 3n$. Since $\mathbf{C} \in \mathbb{R}^{\frac{n(n-1)}{2} \times 3n}$, its rank satisfies $\operatorname{rank}(\mathbf{C}) \le 3n$, and hence
\begin{align}
  \label{Eq18}
  \frac{n(n-7)}{2} \le \dim(\mathcal{N}(\mathbf{C}^\top)) \le \frac{n(n-1)}{2}.
\end{align}

Both lower and upper bounds in \eqref{Eq18} grow quadratically with $n$, indicating that $\mathcal{N}(\mathbf{C}^\top)$ grows rapidly as the number of UAVs increases. Consequently, the set of $\mathbf{q}$ satisfying $\mathbf{q}^\top\mathbf{C}=\mathbf{0}^\top$ becomes increasingly large.

Case~(b) of Lemma~\ref{the2} further requires $\mathbf{q}\ge\mathbf{0}$ and $\mathbf{q}^\top\mathbf{b}<0$. Although these additional constraints restrict $\mathbf{q}$ to a structured subset of $\mathcal{N}(\mathbf{C}^\top)$, the rapid growth of the nullspace dimension enlarges the set of candidate directions that may simultaneously satisfy all conditions. Consequently, internal incompatibility among CBF constraints becomes more likely as $n$ increases, particularly in dense multi-UAV scenarios.
\end{proof}

Lemma~\ref{pro1} indicates that internal incompatibility among CBF constraints becomes increasingly difficult to avoid in large-scale multi-UAV scenarios. In such situations, the CBF-QP may become infeasible, making it difficult to obtain a control input that simultaneously satisfies all CBF constraints and thus limiting the practical applicability of the CBF framework in dense UAV environments. To enhance internal compatibility, we further analyze the feasibility condition in Lemma~\ref{the2} and derive a sufficient condition for internal compatibility. To this end, we reformulate the alternatives in Lemma~\ref{the2} and examine their implications.

\begin{lemma}
  \label{le1}
The CBF constraints \eqref{Eq15c} are compatible if and only if for all vectors $\mathbf{q} \ge \mathbf{0}$, 
at least one of the following holds:
$
\mathbf{q}^\top \mathbf{C} \ne \mathbf{0}^\top
 $ \text{or} $
\mathbf{q}^\top \mathbf{b} \ge 0.
$
\end{lemma}

To enhance compatibility, we aim to prevent the existence of $\mathbf{q}\ge\mathbf{0}$ that simultaneously satisfies $\mathbf{q}^\top\mathbf{C}=\mathbf{0}^\top$ and $\mathbf{q}^\top\mathbf{b}<0$. Since Lemma~\ref{le1} holds trivially for $\mathbf{q}=\mathbf{0}$, we restrict attention to nonzero $\mathbf{q}\ge\mathbf{0}$.

For any nonzero $\mathbf{q}\ge\mathbf{0}$, Lemma~\ref{le1} is automatically satisfied when $\mathbf{q}^\top\mathbf{C}\neq\mathbf{0}^\top$. Hence, a violation of Lemma~\ref{le1} can occur only when $\mathbf{q}^\top\mathbf{C}=\mathbf{0}^\top$. As a result, excluding the existence of any nonzero $\mathbf{q}\ge\mathbf{0}$ with $\mathbf{q}^\top\mathbf{C}=\mathbf{0}^\top$ constitutes a sufficient condition to ensure the validity of Lemma~\ref{le1}. This observation naturally leads us to examine the structural properties of $\mathbf{C}$ under which such a vanishing product may arise.

Since $\mathbf{q}$ can be any nonnegative vector, the existence of a nonzero $\mathbf{q}\ge\mathbf{0}$ satisfying $\mathbf{q}^\top\mathbf{C}=\mathbf{0}^\top$ depends critically on how the blocks of $\mathbf{C}$ are structured and combined. To examine this dependence explicitly, we define $\mathbf{C}_i\in\mathbb{R}^{(n-1)\times3}$ as the submatrix composed of the nonzero blocks of $\mathbf{C}$ associated with A$_i$,
\begin{equation}
\begin{split}
\mathbf{C}_i =&
\big[
\mathbf{k}_{i,1} \; \ldots \; \mathbf{k}_{i,i-1} \;
\mathbf{k}_{i,i+1} \; \ldots \; \mathbf{k}_{i,n}
\big]^\top, \\
\mathbf{q}^\top\mathbf{C} =&
\big[
\mathbf{q}_1^\top\mathbf{C}_1 \; \ldots \; \mathbf{q}_n^\top\mathbf{C}_n
\big],
\label{Eq19}
\end{split}
\end{equation}
where $\mathbf{q}_i \in \mathbb{R}^{n-1}$ denotes the subvector of $\mathbf{q}$ associated with $\mathbf{C}_i$.

\begin{theorem}[Sign-consistency condition]
\label{th1}
For all $i \in [n]$, suppose that each column of $\mathbf{C}_i$ has entries that are either all strictly positive or all strictly negative. Under this assumption, Lemma~\ref{le1} holds, and hence the CBF constraints are compatible.

\end{theorem}

\begin{proof}
Consider an arbitrary vector $\mathbf{q} \ge \mathbf{0}$. The case $\mathbf{q}=\mathbf{0}$ trivially satisfies Lemma~\ref{le1}, and we hence restrict attention to $\mathbf{q}\neq\mathbf{0}$.

Since $\mathbf{q}$ is composed of subvectors $\mathbf{q}_i$ and $\mathbf{q}\neq\mathbf{0}$, there exists at least one index $i\in[n]$ such that $\mathbf{q}_i\neq\mathbf{0}$. Moreover, because $\mathbf{q}\ge\mathbf{0}$, this subvector $\mathbf{q}_i$ contains at least one strictly positive entry.

By construction, $\mathbf{q}_i^\top \mathbf{C}_i$ is a row vector whose entries are the inner products between $\mathbf{q}_i$ and the columns of $\mathbf{C}_i$. Under the assumption of sign-consistency, each column of $\mathbf{C}_i$ has entries that are strictly positive or strictly negative. Therefore, the inner product between $\mathbf{q}_i$ and any column of $\mathbf{C}_i$ is nonzero and has a definite sign whenever $\mathbf{q}_i\ge\mathbf{0}$ and $\mathbf{q}_i\neq\mathbf{0}$. This implies the block row vector $\mathbf{q}_i^\top \mathbf{C}_i \neq \mathbf{0}^\top$.

Since $\mathbf{q}^\top \mathbf{C}$ is composed of $\mathbf{q}_i^\top \mathbf{C}_i$, the existence of at least one  block $\mathbf{q}_i^\top \mathbf{C}_i \neq \mathbf{0}^\top$ implies $\mathbf{q}^\top \mathbf{C} \neq \mathbf{0}^\top$. Therefore, for any $\mathbf{q}\ge\mathbf{0}$ with $\mathbf{q}\neq\mathbf{0}$, the condition in Lemma~\ref{le1} is satisfied.
\end{proof}

To enhance the internal compatibility of the CBF constraints, the sign-consistency condition should be satisfied as much as possible. To this end, we examine the column entries of $\mathbf{C}_i$, namely the row vectors $\mathbf{k}_{i,j}^{\top}$.

According to \eqref{Eq12}, the matrix $\mathbf{W}_i = [\,\mathbf{w}_{i,1},\, \mathbf{w}_{i,2},\, \mathbf{w}_{i,3}\,]$ acts as a right multiplier of $(\mathbf{s}_i - \mathbf{s}_j)^\top$ in the expression of $\mathbf{k}_{i,j}^{\top}$. Consequently, each entry of $\mathbf{k}_{i,j}^{\top}$ is given by the inner product between $(\mathbf{s}_i - \mathbf{s}_j)$ and the corresponding column vector $\mathbf{w}_{i,c}$, where $c \in [3]$. As a result, the sign of each element of $\mathbf{k}_{i,j}^{\top}$ is completely determined by the inner products $(\mathbf{s}_i - \mathbf{s}_j)^\top \mathbf{w}_{i,c}$.
\begin{figure}[b] 
  \centering
  
    \includegraphics[width=0.37\textwidth]{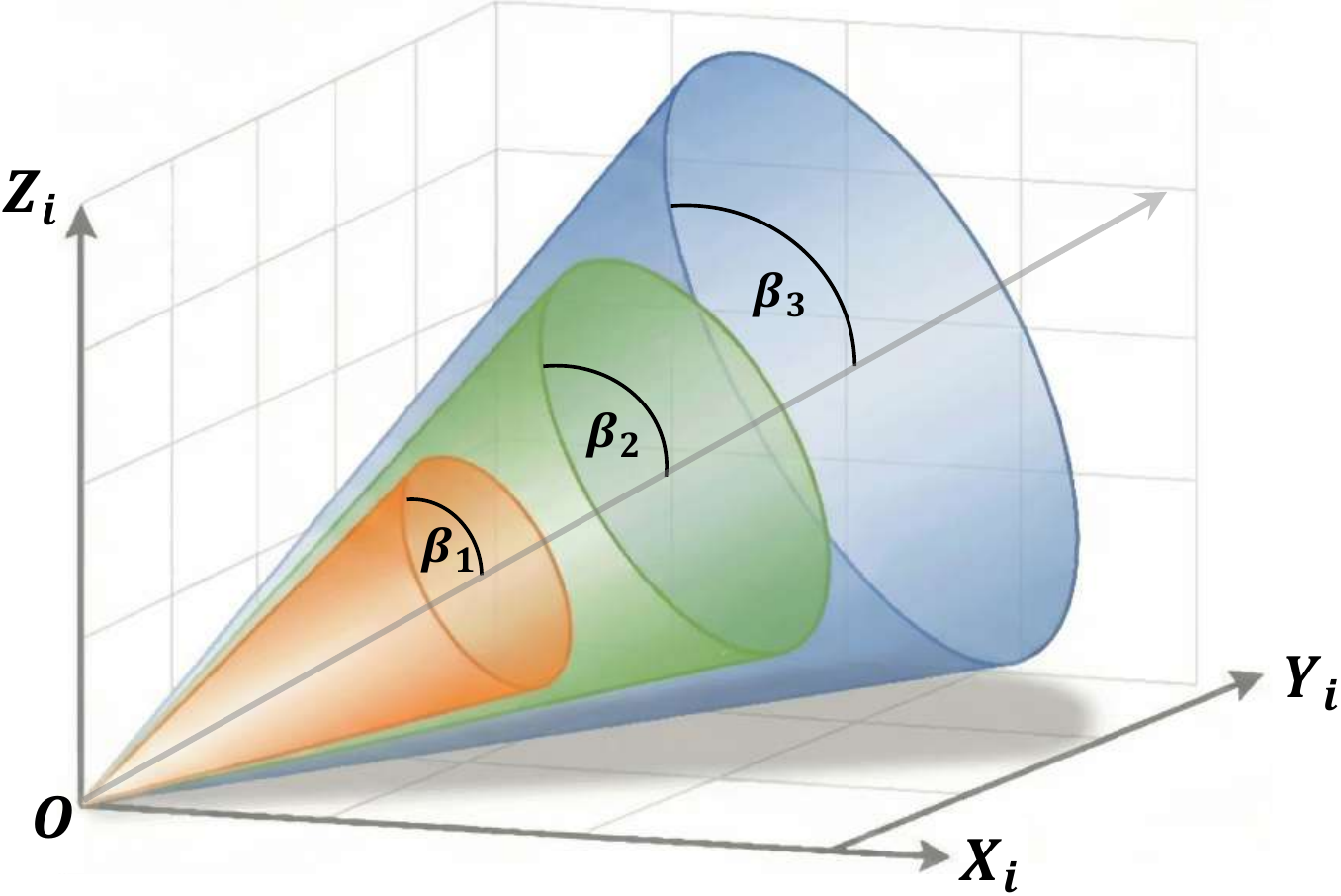}
  
  \caption{The cones of the sign-consistency constraint~\eqref{Eq52} . $\beta_1<\beta_2<\beta_3$ are the half-apex angles of these cones, respectively.}
  \label{fig:conestraint}
\end{figure}
From \eqref{Eq9}, since $\mathbf{W}_i = \mathbf{R}_i \operatorname{diag}(1,\, v_i,\, v_i)$, the column vectors $\mathbf{w}_{i,c}$ are mutually orthogonal and hence form an orthogonal basis of $\mathbb{R}^3$. Accordingly, a local coordinate frame $\mathrm{O}\!-\!\mathrm{X}_i\mathrm{Y}_i\mathrm{Z}_i$ can be defined with $\mathbf{s}_i$ as the origin, whose $\mathrm{X}$-, $\mathrm{Y}$-, and $\mathrm{Z}$-axes are aligned with $\mathbf{w}_{i,1}$, $\mathbf{w}_{i,2}$, and $\mathbf{w}_{i,3}$, respectively. Therefore, the signs of the elements of $\mathbf{k}_{i,j}^{\top}$ are uniquely determined by the octant of the local frame $\mathrm{O}\!-\!\mathrm{X}_i\mathrm{Y}_i\mathrm{Z}_i$ in which the relative vector $(\mathbf{s}_i - \mathbf{s}_j)$ lies.

Since the sign-consistency condition requires each column of $\mathbf{C}_i$ to have identical signs, it follows that, for all $j \in [n] \backslash \{i\}$, the relative vectors $(\mathbf{s}_i - \mathbf{s}_j)$ must lie within the same octant of $\mathrm{O}\!-\!\mathrm{X}_i\mathrm{Y}_i\mathrm{Z}_i$. When this condition is satisfied, Lemma~\ref{le1} holds, and hence the CBF constraint $\mathbf{C}\mathbf{u} \leq \mathbf{b}$ is internally compatible.

To enforce this structural requirement, we construct a sign-consistency constraint for A$_i$ by approximating the target octant using a conic region. This constraint confines all relevant vectors to a common cone aligned with the desired axis direction, which is formulated as
\begin{align}
\label{Eq20}
\mathbf{d}_i^{\top}\mathbf{\widetilde{W}}_i^{\top}
\frac{
\mathbf{s}_i + \dot{\mathbf{s}}_i - \mathbf{s}_j - \dot{\mathbf{s}}_j
}{
\left\|\mathbf{s}_i + \dot{\mathbf{s}}_i - \mathbf{s}_j - \dot{\mathbf{s}}_j\right\|
}
\geq \cos\beta, \quad j \in [n] \backslash \{i\},
\end{align}
where $\beta \in (0, \pi/2)$ denotes the half-apex angle of the cone. The matrix $\mathbf{\widetilde{W}}_i = [\mathbf{\widetilde{w}}_{i,1}, \mathbf{\widetilde{w}}_{i,2}, \mathbf{\widetilde{w}}_{i,3}]$ is composed of normalized orthogonal vectors defining the local frame, and $\mathbf{d}_i$ denotes the axis direction of the cone, given by
\begin{align}
\label{Eq51} 
\mathbf{d}_i=\frac{\sqrt{3}}{3}\,\boldsymbol{\rm sign}\left(\mathbf{\widetilde{W}}_i^{\top}\left(\mathbf{p}^g_i-\mathbf{p}_i\right)\right),
\end{align}
where $\boldsymbol{\rm sign}(\cdot)$ denotes the element-wise sign function, and $\mathbf{p}^g_i$ represents the nominal goal position of A$_i$.

However, this constraint cannot be directly incorporated into the QP.
According to~\eqref{Eq8}, the state derivative $\dot{\mathbf{s}}$ explicitly
depends on the control input, which causes the normalization term in the
denominator of \eqref{Eq20} to become control-dependent.
As a result, the constraint becomes nonlinear in the control input and
thus incompatible with the standard QP formulation.

To derive a tractable constraint, the denominator is approximated
using the velocity $\mathbf{v}$ instead of $\dot{\mathbf{s}}$, thereby eliminating its dependence on the control input. Additionally, a nonnegative slack variable $\epsilon_{i,j}$ is introduced to allow for limited violations of the constraint.
The resulting sign-consistency constraint implemented in the QP
is formulated as
\begin{align}
\label{Eq52}
\mathbf{d}_i^{\top}\mathbf{\widetilde{W}}_i^{\top}
\frac{
\mathbf{s}_i + \dot{\mathbf{s}}_i - \mathbf{s}_j - \dot{\mathbf{s}}_j
}{
\left\|\mathbf{s}_i + \mathbf{v}_i - \mathbf{s}_j - \mathbf{v}_j\right\|
}
+ \epsilon_{i,j}
\geq \cos\beta, \quad j \in [n]\backslash\{i\},
\end{align}

\subsection{Decentralized CBF-QP with Sign-Consistency Constraint}

In this subsection, the sign-consistency constraint is incorporated into the decentralized CBF-QP framework. In the decentralized collision avoidance formulation, A$_i$ has access only to the states of its neighboring UAVs, indexed by the set $\mathcal{A}_i \subseteq [n]\backslash\{i\}$, but not to their control inputs. As a result, the original CBF and sign-consistency constraints, which explicitly depend on $\mathbf{u}_j$ and $\dot{\mathbf{s}}_j$, are no longer directly applicable and must be reformulated.
\begin{figure}[!tb]
  \centering
    \includegraphics[width=0.47\textwidth]{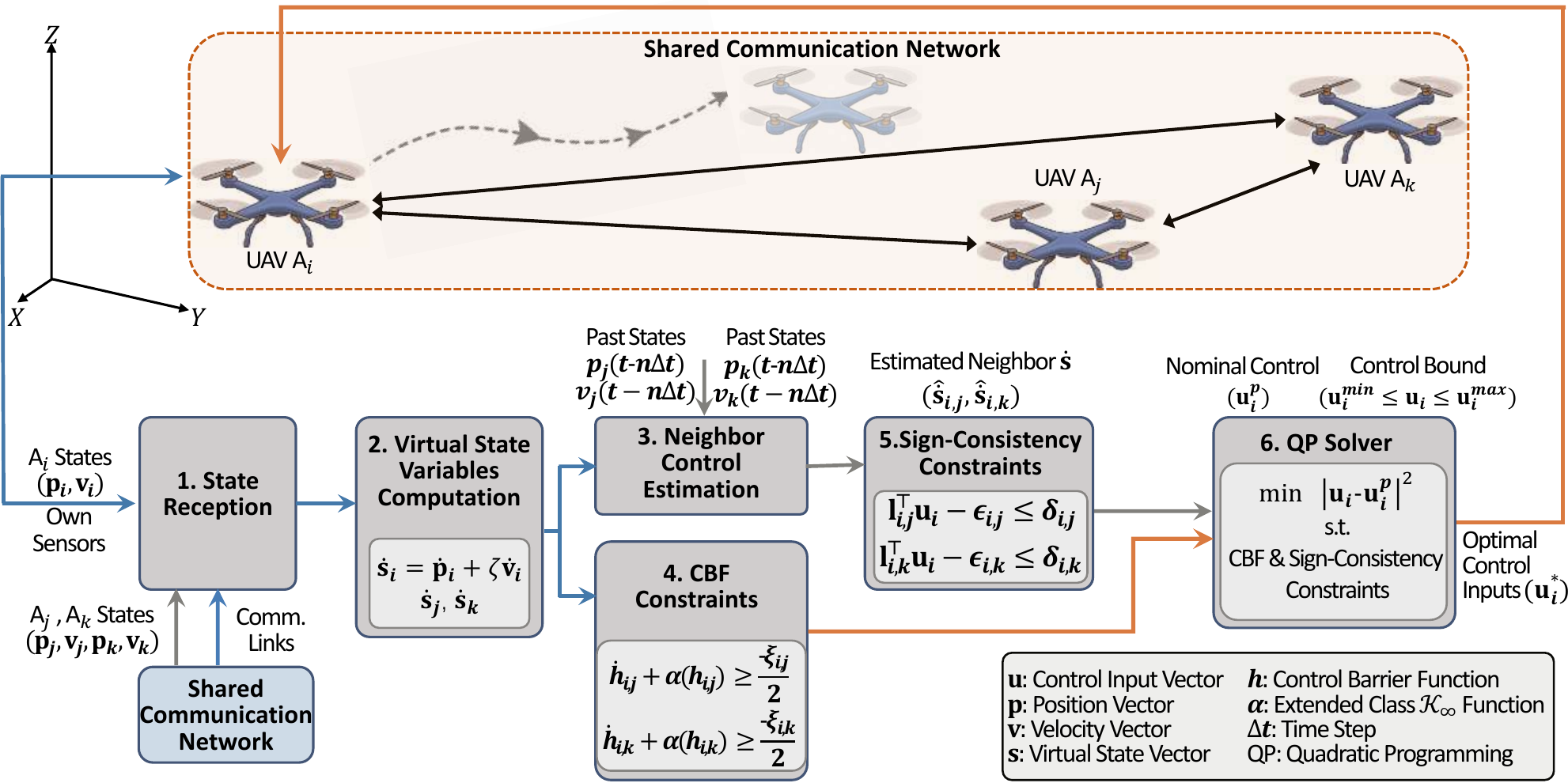}
  
  \caption{Overview of the proposed FECBF method.}
  \label{fig:overview}
\end{figure}
Following the commonly adopted assumption of equal collision avoidance responsibility between UAVs, the decentralized CBF constraint is reformulated from~\eqref{Eq11} as

\begin{align}
\begin{split}
\label{Eq23} 
-\mathbf{k}^{\top}_{i,j}\mathbf{u}_i\leq&\frac{\xi_{i,j}}{2}, \quad  j\in \mathcal{A}_i.
\end{split}
\end{align}

To reformulate the sign-consistency constraint, $\dot{\mathbf{s}}_j$ is replaced with a
worst-case estimate inferred from historical data.
Let $\mathcal{S}_{i,j}$ denote the set of historical states of A$_j$ known to
A$_i$.
Using finite-difference approximations, A$_i$ reconstructs a set of admissible control input sequences consistent with the observed state evolution. The estimate $\widehat{\dot{\mathbf{s}}}_{i,j}$ is obtained by selecting the control input sequence that maximizes $\mathbf{d}_i^\top \widetilde{\mathbf{W}}_i^\top \dot{\mathbf{s}}_j$, corresponding to the worst-case variation of the sign-consistency constraint \eqref{Eq52}. Therefore, the decentralized sign-consistency constraint is reformulated as
\begin{align}
\begin{split}
\label{Eq24} 
\mathbf{l}^{\top}_{i,j}\mathbf{u}_i - \epsilon_{i,j} \leq &\ \delta_{i,j}, \quad j \in \mathcal{A}_i,
\end{split}
\end{align}
where
\begin{align}
\begin{split}
\label{Eq25} 
\mathbf{l}_{i,j} = &\ \frac{-\zeta \mathbf{d}_i^\top \widetilde{\mathbf{W}}_i^\top \mathbf{W}_i}{\left\|\mathbf{s}_i + \mathbf{v}_i - \mathbf{s}_j - \mathbf{v}_j\right\|}, \\
\delta_{i,j} = &\ \frac{\mathbf{d}_i^\top \widetilde{\mathbf{W}}_i^\top\left(\mathbf{s}_i + \dot{\mathbf{s}}_i - \mathbf{s}_j - \widehat{\dot{\mathbf{s}}}_{i,j}\right)}{\left\|\mathbf{s}_i + \mathbf{v}_i - \mathbf{s}_j - \mathbf{v}_j\right\|} - \cos\beta,
\end{split}
\end{align}

Consequently, the decentralized feasibility-enhanced control barrier function quadratic program (FECBF-QP) is formulated as
\begin{subequations}
\label{Eq26}
\begin{align}
\mathop{\arg\min}_{\mathbf{u}_i,\, \boldsymbol{\upepsilon}_i} \quad 
    & \lVert \mathbf{u}_i - \mathbf{u}_i^{p} \rVert^2 + \lambda \lVert \boldsymbol{\upepsilon}_i \rVert^2, \label{Eq3a}\\
\text{s.t.} \quad 
    & \mathbf{u}_i^{\min} \leq \mathbf{u}_i \leq \mathbf{u}_i^{\max}, 
      \qquad \boldsymbol{\upepsilon}_i \geq \boldsymbol{0}, \label{Eq3b}\\
    & -\mathbf{k}^{\top}_{i,j}\mathbf{u}_i \leq \tfrac{\xi_{i,j}}{2}, 
      \quad \forall j \in \mathcal{A}_i, \label{Eq3c}\\
    & \mathbf{l}_{i,j}^{\top} \mathbf{u}_i - \epsilon_{i,j} \leq \delta_{i,j}, 
      \quad \forall j \in \mathcal{A}_i, \label{Eq3e}
\end{align}
\end{subequations}
where $\boldsymbol{\upepsilon}_i =  \left( \epsilon_{i,j} \right)_{j \in \mathcal{A}_i}$ denotes the vector of slack variables, and $\lambda > 0$ denotes the weight of the slack term. Based on the formulated control barrier functions and the sign-consistency constraints, each UAV executes  the optimal avoidance maneuver derived from this model. The overall framework of the proposed method is illustrated in Fig.~\ref{fig:overview}.

\section{EXPERIMENTS}
\label{se3}
This section presents the experimental validation of the proposed method. We begin by describing the simulation setup and configurations, followed by a comparison with CBF-based baseline methods. The robustness of the method to time delays is then evaluated through dedicated simulations. Finally, the section concludes with real-world experiments.

\subsection{Simulation Configurations and Parameters}
\label{seSC}
All simulations are conducted in MATLAB R2019a on a workstation equipped with an Intel Core i7-12700KF processor (12 cores) and 32~GB of RAM. The key experimental parameters are summarized in Table~\ref{tab:exp_params}. For each UAV, the maximum speed $v_i^{\max}$ is randomly sampled from the range $[2, 3]$ $\mathrm{m/s}$ at the beginning of each trial and remains fixed throughout the simulation.

\begin{table}[h]
    \centering
    \caption{Key experimental parameters used in simulations. Angular quantities are given in radians.}
    \label{tab:exp_params}
    \begin{tabular}{llllll}
        \toprule
        Para.            & Val.          & Para.            & Val.   & Para.            & Val.       \\
        \midrule
        $\theta^{\max}_i$         & ${\pi}/{2}$   & $v^{\max}_i$      & $[2, 3]\,\mathrm{m/s}$   & $\gamma^{\max}_i$      & $\pi/36$  \\
        $\theta^{\min}_i$         & $-{\pi}/{2}$  & $v^{\min}_i$     & ${v^{\max}_i}/{4}$   & $\gamma^{\min}_i$     & $-\pi/36$                \\
        $\psi^{\max}_i$  & $2{\pi}$& $a^{\max}_i$     & $1\,\mathrm{m/s}$  & $\omega^{\max}_i$     & $\pi/18$             \\
        $\psi^{\min}_i$ & $0$& $a^{\min}_i$     & $-1\,\mathrm{m/s}$ & $\omega^{\min}_i$     & $-\pi/18$\\    
        $\zeta$ & $0.5$& $\kappa$     & $0.08$ & $\beta$     & $7\pi/24$\\    
        $\lambda$ & $3$\\  
        \bottomrule
    \end{tabular}
    \vspace{-1mm}
\end{table}

\subsection{Simulation Setup and Evaluation Metrics}

The performance of the proposed method is compared against the following established CBF-based approaches in simulation:
\begin{itemize}
    \item \textit{DRCBF}~\cite{jankovic2023multiagent}: A decentralized CBF-based method where each UAV assumes half of the responsibility for collision avoidance by symmetrically splitting the safety constraints.
    \item \textit{Velocity Obstacle CBF (VOCBF)}~\cite{roncero2025multi}: A CBF formulation augmented with velocity obstacle (VO) constraints that restrict the UAV's velocity to remain outside the VO, thereby preventing collisions.
\end{itemize}

Four evaluation metrics are used to assess the performance of all methods:
\begin{itemize}
    \item \textit{Success Rate (SR)}: For each trial, SR is defined as the ratio of UAVs that successfully reach their destinations without collision to the total number of UAVs in that trial.
    \item \textit{Infeasibility Count (IC)}: The total number of control steps during which the CBF optimization problem becomes infeasible for a UAV throughout its entire flight.
    \item \textit{Arrival Time (AT)}: The time required for a UAV to reach its destination.
    \item \textit{Computational Time (CT)}: The time required to compute the control input at each time step.
\end{itemize}

\begin{figure}[!b] 
  \centering
    \subfloat[]{\includegraphics[width=0.15\textwidth]{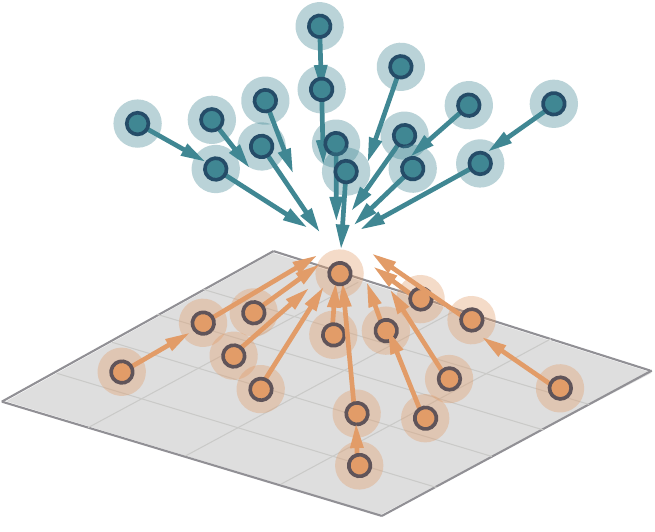}}
    \subfloat[]{\includegraphics[width=0.13\textwidth]{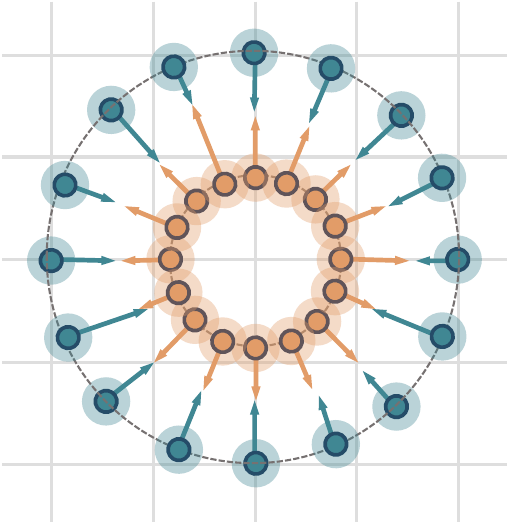}}
    \subfloat[]{\includegraphics[width=0.15\textwidth]{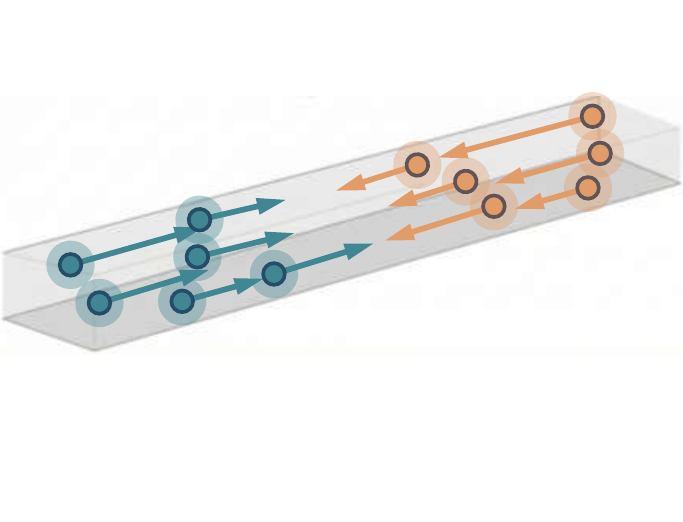}}
  \caption{Illustration of three representative and challenging simulation scenarios used for performance evaluation. The gray shaded region denotes the ground plane, while the points with halo and arrows represent the positions and velocity vectors of the UAVs, respectively. (a) Convergence scenario; (b) Dual-circle scenario, depicted from a top view; (c) Head-on scenario.}
  \label{fig:trajectories}
\end{figure}

\begin{table*}[!]
\renewcommand{\arraystretch}{1.3}
\caption{Performance Comparison of Different Methods Across Scenarios}
\label{table:comparison_results}
\centering
\scriptsize 
\setlength{\tabcolsep}{1.5pt} 

\begin{tabular}{
    >{\raggedright\arraybackslash}p{40pt} 
    >{\centering\arraybackslash}p{13pt}   
    *{4}{ >{\centering\arraybackslash}p{26pt} 
          >{\centering\arraybackslash}p{26pt} 
          >{\centering\arraybackslash}p{26pt} } 
}
\toprule
\multirow{2}{*}{Scenario} & \multirow{2}{*}{$n$} & \multicolumn{3}{c}{SR [\%] $\uparrow$} & \multicolumn{3}{c}{IC  $\downarrow$} & \multicolumn{3}{c}{AT [s] $\downarrow$} & \multicolumn{3}{c}{CT [ms] $\downarrow$} \\
\cmidrule(lr){3-5} \cmidrule(lr){6-8} \cmidrule(lr){9-11} \cmidrule(lr){12-14}
& & Ours & VOCBF & DRCBF & Ours & VOCBF & DRCBF & Ours & VOCBF & DRCBF & Ours & VOCBF & DRCBF \\
\midrule
& 50 & \textbf{97.98} & 96.72 & 95.92 & \textbf{96.58} & 114.50 & 121.11 & \textbf{377.28} & 382.14 & 384.70 & 1.10 & 1.14 & \textbf{0.96} \\
Convengence & 100 & \textbf{96.72} & 95.56 & 93.66 & \textbf{147.73} & 182.82 & 184.41 & \textbf{408.32} & 419.69 & 418.65 & 1.32 & 1.61 & \textbf{1.08} \\
& 150 & \textbf{95.89} & 94.01 & 91.59 & \textbf{193.51} & 240.49 & 242.01 & \textbf{435.45} & 448.91 & 447.32 & 1.50 & 1.89 & \textbf{1.21} \\
\hline
& 50 & \textbf{100.00} & 93.02 & 93.16 & \textbf{75.31} & 141.44 & 139.78 & \textbf{332.99} & 337.39 & 336.86 & 0.86 & 0.80 & \textbf{0.62} \\
Dual-circle & 100 & \textbf{100.00} & 93.56 & 93.74 & \textbf{74.59} & 141.45 & 139.79 & \textbf{332.97} & 337.38
 & 336.85& 1.10 & 1.05 & \textbf{0.96} \\
& 150 & \textbf{100.00} & 95.72 & 94.09 & \textbf{68.73} & 138.05 & 139.82 & \textbf{332.52} & 336.86 & 336.86 & 1.33 & 1.12 & \textbf{0.94} \\
\hline
& 50 & \textbf{98.00} & 97.64 & 97.90 & \textbf{58.89} & 66.84 & 64.93 & 349.04 & 348.87 & \textbf{347.65} & 1.20 & 1.49 & \textbf{1.07}\\
Head-on & 100 & \textbf{96.28} & 94.11 & 93.71 & \textbf{111.50} & 145.40 & 141.21 & 382.06 & 382.75 & \textbf{381.80} & 1.39 & 1.82 & \textbf{1.20} \\
& 150 & \textbf{95.57} & 89.45 & 89.14 & \textbf{163.82} & 218.59 & 215.35 & 414.49 & 412.53 & \textbf{411.58} & 1.57 & 2.17 & \textbf{1.39} \\
\bottomrule
\addlinespace[0.5ex]
\multicolumn{14}{l}{\scriptsize Ours: FECBF; SR: Success Rate; IC: Infeasibility Count; AT: Arrival Time; CT: Computational Time; \textbf{Bold} indicates best performance.}
\end{tabular}
\vspace{-2mm}
\end{table*}

Each method is evaluated under the following three scenarios:
\begin{itemize}
    \item \textit{Convergence Scenario}: $n$ UAVs are initialized on collision courses toward a common waypoint $\mathbf{p}_{c} = [1000, 1000, 250]^\top\,\mathrm{m}$. The initial distance of each UAV from the waypoint is determined according to its $v_i^{\max}$ such that, in the absence of collision avoidance, all UAVs would arrive at $\mathbf{p}_{c}$ simultaneously at $t = 150\,\mathrm{s}$.
    
    \item \textit{Dual-circle Scenario}: $n$ UAVs are distributed on two concentric circles with radii of $400\,\mathrm{m}$ and $600\,\mathrm{m}$ at an altitude of $200\,\mathrm{m}$, with $n/2$ UAVs assigned to each circle. UAVs on the outer circle fly toward the center, while those on the inner circle fly radially outward.
    
    \item \textit{Head-on Scenario}: Two groups of $n/2$ UAVs are initialized on opposite sides of a region, with their lateral positions confined within a width of $200\,\mathrm{m}$. The two groups fly toward each other to induce dense head-on encounters.
\end{itemize}

For all scenarios, each UAV is assigned a destination along its initial flight direction. The travel distance is determined by its maximum speed $v_i^{\max}$ and a fixed mission duration of $300\,\mathrm{s}$. The illustration of these scenarios is shown in Fig.~\ref{fig:trajectories}.

\subsection{Comparison with Baselines}
\label{seRA}
We evaluate the performance of our method in three scenarios involving 50, 100, and 150 UAVs. To ensure statistical significance, 100 Monte Carlo simulations are conducted for each case. The quantitative results are summarized in Table~\ref{table:comparison_results}.

As shown in Table~\ref{table:comparison_results}, our method consistently achieves the highest SR in all scenarios and UAV counts. In particular, in the dual-circle scenario, our approach maintains a perfect SR of 100$\%$ for all UAV counts, while both VOCBF and DRCBF exhibit noticeable performance degradation. In the convergence and head-on scenarios, although the SR of all methods decreases as the number of UAVs increases, our method demonstrates superior scalability by consistently outperforming the baselines, with the performance gap becoming more pronounced in denser settings.
\begin{figure}[!t] 
  \centering
    \includegraphics[width=0.46\textwidth]{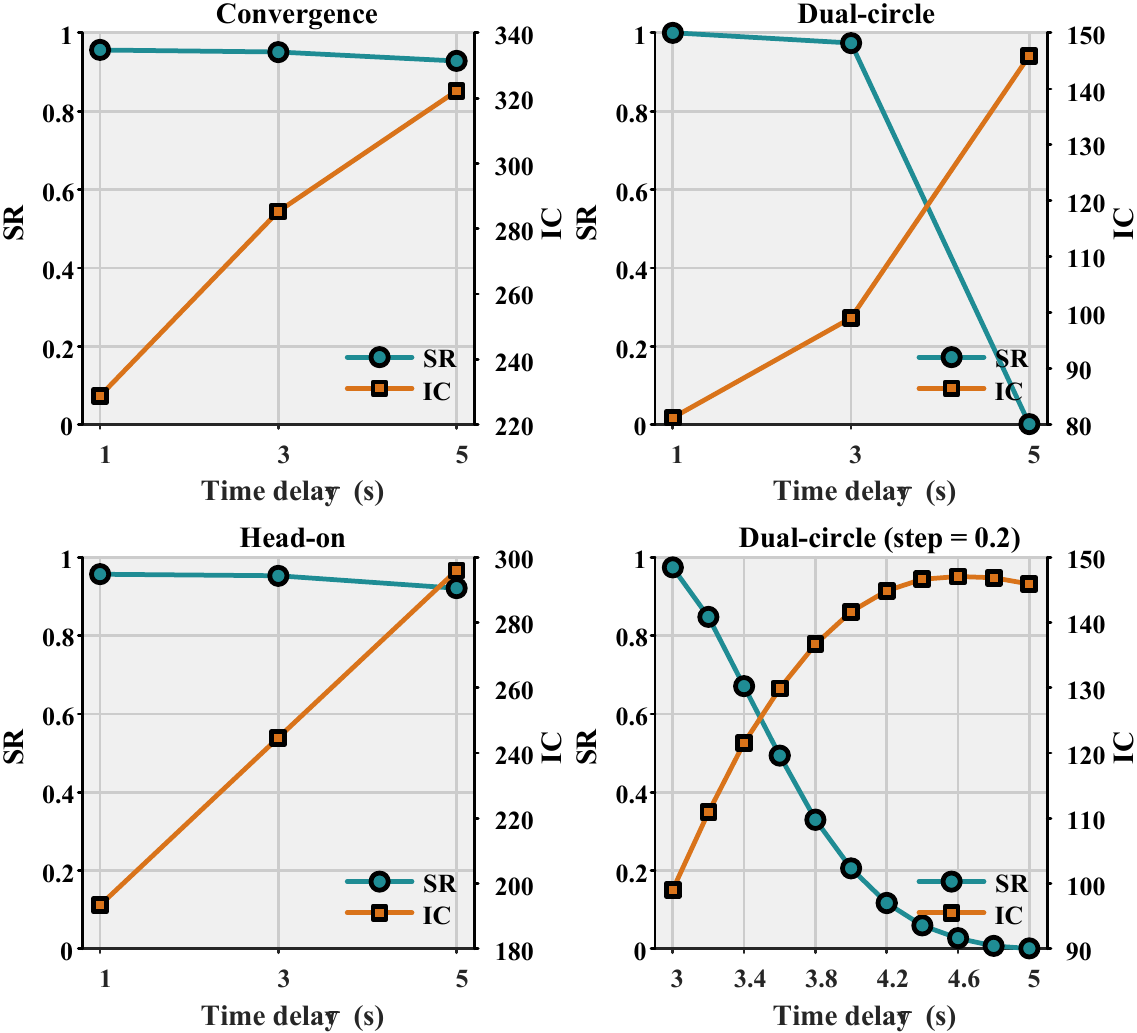}
  \caption{Time delay simulation results for all three scenarios with $n = 150$. The bottom-right subfigure shows the results for the dual-circle scenario with $\tau$ ranging from $3$ s to $5$ s at intervals of $0.2$ s, while the remaining subfigures show results for $\tau \in \{1, 3, 5\}$ s.}
  \label{fig:timedelay}
\end{figure}
The improved SR is closely associated with the reduced IC. Our method maintains the lowest IC across all scenarios and UAV numbers, indicating a significantly enhanced feasibility of the underlying CBF-QP. This advantage becomes increasingly evident in the convergence and head-on scenarios as the UAV density grows, highlighting the effectiveness of the proposed sign-consistency constraint in mitigating constraint conflicts. The strong correlation between the trends of IC and SR indicates that improving feasibility is crucial for  improving the overall reliability of CBF-based collision avoidance. 

In addition to safety and feasibility, our method exhibits favorable performance in terms of AT. Specifically, it achieves the shortest AT in nearly all cases in the convergence and dual-circle scenarios. In particular, in the convergence scenario, the AT is reduced by more than 10 seconds compared to the baselines as the number of UAVs increases, indicating that improved collision avoidance performance is achieved without sacrificing operational efficiency.

Regarding CT, our method and VOCBF incur slightly higher per-step computation costs than DRCBF due to the inclusion of sign-consistency constraints and velocity obstacle constraints, respectively. Nevertheless, the CT of all evaluated methods remains within the same order of magnitude across all UAV counts. Importantly, the computational overhead of our method consistently satisfies real-time requirements, demonstrating its practical applicability in dense multi-UAV environments.

\subsection{Evaluation of Time Delay Robustness}
\label{seRB}
Time delay is a key factor affecting the effectiveness of collision avoidance algorithms in practical applications. To assess the robustness of our method under latency, simulations are conducted in all three scenarios with time delays $\tau \in \{1,3,5\}$~s, focusing on the most challenging setting with UAV count $n=150$.

As shown in Fig.~\ref{fig:timedelay}, the results demonstrate a graceful performance degradation of our method as the delay increases in the convergence and head-on scenarios. Specifically, SR decreases moderately from approximately $0.96$ to above $0.92$ as the delay $\tau$ increases from $1~\mathrm{s}$ to $5~\mathrm{s}$, accompanied by a monotonic but controlled increase in IC.
In these scenarios, potential conflicts do not evolve in a highly time-critical manner, leaving sufficient collision avoidance margin even under delayed information. Consequently, both SR and IC remain at favorable levels under moderate communication delays. 

In contrast, the dual-circle scenario exhibits higher sensitivity to delay, as UAVs move directly toward each other and the available reaction time is reduced. Although SR drops noticeably under large delays, the corresponding increase in IC remains moderate because infeasible CBF-QP instances mainly occur during a short time-critical interaction phase. Once the opposing UAVs pass through this region, further conflicts rarely arise and the CBF-QP quickly returns to a feasible region.

\begin{figure}[!tb] 
  \centering
    \includegraphics[width=0.44\textwidth]{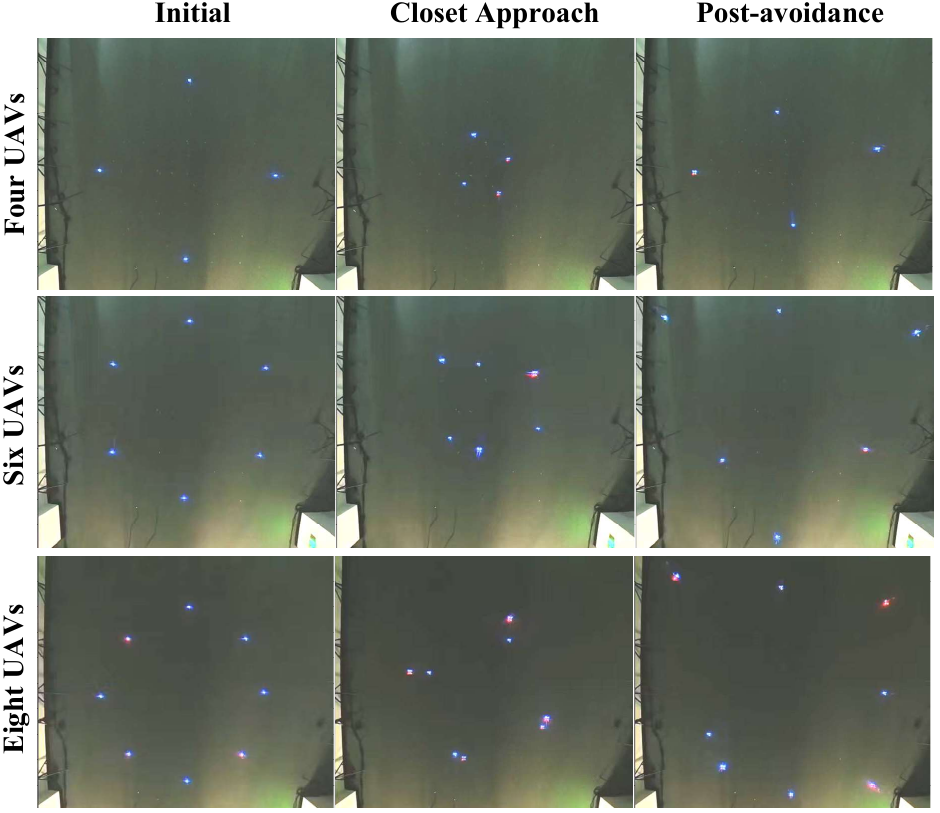}
  \caption{Snapshots of real-world experiments with four, six, and eight UAVs.
Each row corresponds to a different UAV counts, while columns represent the initial configuration, the moment of closest approach during the collision avoidance, and the post-avoidance configuration.}
  \label{fig:real-world}
\end{figure}

\subsection{Real-world Experiments}
To validate the effectiveness of the proposed method in practice, we conduct real-world experiments with Crazyswarm platforms and NOKOV motion capture systems. In the initial configuration, the UAVs are uniformly distributed in the horizontal plane, with half positioned at an altitude of 1.5 m and the other half at 0.9 m. They travel toward their opposite positions, resulting in a convergence scenario. The UAVs successfully avoid collisions using the proposed method. Representative snapshots are shown in Fig.~\ref{fig:real-world}.

\section{CONCLUSION}
\label{se5}
In this letter, we improve the feasibility of CBF-based multi-UAV collision avoidance in dense scenarios. We analyze the internal compatibility among multiple CBF constraints to reveal a fundamental bottleneck limiting the feasibility of the CBF-QP. Based on this analysis, a sign-consistency constraint is proposed and incorporated into a decentralized CBF-QP framework to enhance feasibility. Simulation results demonstrate reduced infeasibility and improved collision avoidance performance, while real-world experiments validate the practical applicability of the proposed method. Future work will investigate compatibility analysis under uncertainty.



\bibliography{IEEEabrv,TRO3}

@article{ge2025multi,
  title={Multi-uav search and rescue in wilderness using smart agent-based probability models},
  author={Ge, Zijian and Jiang, Jingjing and Coombes, Matthew},
  journal={IEEE Transactions on Aerospace and Electronic Systems},
  volume={62},
  pages={1649--1662},
  year={2026},
  publisher={IEEE}
}

@article{cao2025cooperative,
  title={Cooperative Aerial Robot Inspection Challenge: A Benchmark for Heterogeneous Multi-Uncrewed-Aerial-Vehicle Planning and Lessons Learned},
  author={Cao, Muqing and Nguyen, Thien-Minh and Yuan, Shenghai and Anastasiou, Andreas and Zacharia, Angelos and Papaioannou, Savvas and Kolios, Panayiotis and Panayiotou, Christos G and Polycarpou, Marios M and Xu, Xinhang and others},
  journal={IEEE Robotics \& Automation Magazine},
  year={2025},
  publisher={IEEE}
}

@article{zhang2025fast,
  title={A Fast Multi-UAV Assistance Positioning Architecture for Large-Scale Farming Operations},
  author={Zhang, Jingyao and Chen, Haihua and Dai, Feng and Wang, Yancong and Li, Heyang and Zhang, Yucheng},
  journal={IEEE Internet of Things Journal},
  year={2025},
  volume={12},
  number={18},
  pages={37645--37658},
  publisher={IEEE}
}

@article{zhong20253d,
  title={3D RVO-Enhanced Multi-Agent Deep Reinforcement Learning for Collision Avoidance in Urban Structured Airspace},
  author={Zhong, Gang and Yupu, Liu and Sen, Du and Fei, Wang and Jinlun, Zhou and Zhang, Honghai},
  journal={Aerospace Science and Technology},
  pages={110378},
  year={2025},
  publisher={Elsevier}
}

@article{cattai2025multi,
  title={Multi-uav reinforcement learning with realistic communication models: Recent advances and challenges},
  author={Cattai, Tiziana and Frattolillo, Francesco and Lacava, Andrea and Raut, Prasanna and Simonjan, Jennifer and D'Oro, Salvatore and Melodia, Tommaso and Vinogradov, Evgenii and Natalizio, Enrico and Colonnese, Stefania and others},
  journal={IEEE Open Journal of Vehicular Technology},
  year={2025},
  pages={2067--2081},
  volume={6},
  publisher={IEEE}
}

@inproceedings{roncero2025multi,
  title={Multi-agent obstacle avoidance using velocity obstacles and control barrier functions},
  author={Roncero, Alejandro S{\'a}nchez and Muchacho, Rafael I Cabral and {\"O}gren, Petter},
  booktitle={2025 IEEE International Conference on Robotics and Automation (ICRA)},
  pages={6638--6644},
  year={2025},
  organization={IEEE}
}

@article{jankovic2023multiagent,
  title={Multiagent systems with CBF-based controllers: Collision avoidance and liveness from instability},
  author={Jankovic, Mrdjan and Santillo, Mario and Wang, Yan},
  journal={IEEE Transactions on Control Systems Technology},
  volume={32},
  number={2},
  pages={705--712},
  year={2023},
  publisher={IEEE}
}

@article{huang2025dynamic,
  title={Dynamic collision avoidance using velocity obstacle-based control barrier functions},
  author={Huang, Jihao and Zeng, Jun and Chi, Xuemin and Sreenath, Koushil and Liu, Zhitao and Su, Hongye},
  journal={IEEE Transactions on Control Systems Technology},
  year={2025},
  volume={33},
  number={5},
  pages={1601--1615},
  publisher={IEEE}
}

@inproceedings{huriot2025safe,
  title={Safe decentralized multi-agent control using black-box predictors, conformal decision policies, and control barrier functions},
  author={Huriot, Sacha and Sibai, Hussein},
  booktitle={2025 IEEE International Conference on Robotics and Automation (ICRA)},
  pages={7445--7451},
  year={2025},
  organization={IEEE}
}

@inproceedings{liu2025safety,
  title={Safety-critical planning and control for dynamic obstacle avoidance using control barrier functions},
  author={Liu, Shuo and Mao, Yihui and Belta, Calin A},
  booktitle={2025 American Control Conference (ACC)},
  pages={348--354},
  year={2025},
  organization={IEEE}
}

@article{yang2025geometry,
  title={Geometry-Based Cooperative Conflict Resolution for Multi-UAV Combining Heading and Speed Control},
  author={Yang, Jian and Zhang, Kaixin and Zhong, Qishen and Zhang, Lidong},
  journal={IEEE Transactions on Consumer Electronics},
  volume={71},
  number={1},
  pages={945--958},
  year={2025},
  publisher={IEEE}
}

@article{qin2023srl,
  title={SRL-ORCA: A socially aware multi-agent mapless navigation algorithm in complex dynamic scenes},
  author={Qin, Jianmin and Qin, Jiahu and Qiu, Jiaxin and Liu, Qingchen and Li, Man and Ma, Qichao},
  journal={IEEE Robotics and Automation Letters},
  volume={9},
  number={1},
  pages={143--150},
  year={2023},
  publisher={IEEE}
}

@inproceedings{arul2021v,
  title={V-rvo: Decentralized multi-agent collision avoidance using voronoi diagrams and reciprocal velocity obstacles},
  author={Arul, Senthil Hariharan and Manocha, Dinesh},
  booktitle={2021 IEEE/RSJ International Conference on Intelligent Robots and Systems (IROS)},
  pages={8097--8104},
  year={2021},
  organization={IEEE}
}

@article{yang2023collision,
  title={Collision avoidance for autonomous vehicles based on MPC with adaptive APF},
  author={Yang, Hongjiu and He, Yongqi and Xu, Yang and Zhao, Hai},
  journal={IEEE Transactions on Intelligent Vehicles},
  volume={9},
  number={1},
  pages={1559--1570},
  year={2023},
  publisher={IEEE}
}

@article{qian2025collision,
  title={Collision Avoidance Control for Autonomous Driving with Multiple Dynamic Obstacles in IoV: A Prediction-Enhanced APF-Based Approach},
  author={Qian, Zenghui and Chen, Ruoyang and Yi, Changyan and Zhai, Xiangping and Chen, Bing},
  journal={IEEE Internet of Things Journal},
  volume={12},
  number={13},
  pages={24968--24984},
  year={2025},
  publisher={IEEE}
}

@article{han2025deep,
  title={A Deep Reinforcement Learning Method for Collision Avoidance with Dense Speed-Constrained Multi-UAV},
  author={Han, Jiale and Zhu, Yi and Yang, Jian},
  journal={IEEE Robotics and Automation Letters},
  volume={10},
  number={3},
  pages={2152--2159},
  year={2025},
  publisher={IEEE}
}

@article{kuo2025deep,
  title={Deep reinforcement learning--based collision avoidance strategy for multiple unmanned aerial vehicles},
  author={Kuo, Ping-Huan and Chen, Kuan-Lin and Lin, Yu-Sian and Chiu, Yu-Chih and Peng, Chao-Chung},
  journal={Engineering Applications of Artificial Intelligence},
  volume={160},
  pages={111862},
  year={2025},
  publisher={Elsevier}
}

@article{xiao2022sufficient,
  title={Sufficient conditions for feasibility of optimal control problems using control barrier functions},
  author={Xiao, Wei and Belta, Calin A and Cassandras, Christos G},
  journal={Automatica},
  volume={135},
  pages={109960},
  year={2022},
  publisher={Elsevier}
}

@article{isaly2024feasibility,
  title={On the feasibility and continuity of feedback controllers defined by multiple control barrier functions},
  author={Isaly, Axton and Ghanbarpour, Masoumeh and Sanfelice, Ricardo G and Dixon, Warren E},
  journal={IEEE Transactions on Automatic Control},
  volume={69},
  number={11},
  pages={7326--7339},
  year={2024},
  publisher={IEEE}
}

@article{ames2016control,
  title={Control barrier function based quadratic programs for safety critical systems},
  author={Ames, Aaron D and Xu, Xiangru and Grizzle, Jessy W and Tabuada, Paulo},
  journal={IEEE Transactions on Automatic Control},
  volume={62},
  number={8},
  pages={3861--3876},
  year={2016},
  publisher={IEEE}
}

@book{bertsimas1997introduction,
  title={Introduction to linear optimization},
  author={Bertsimas, Dimitris and Tsitsiklis, John N},
  volume={6},
  year={1997},
  publisher={Athena scientific Belmont, MA}
}
\bibliographystyle{IEEEtran}

\end{document}